\documentclass[11pt]{article}

\usepackage[preprint,nonatbib]{neurips_2025}
\usepackage{amsmath,amssymb,amsfonts,amsthm,mathtools}
\newtheorem{theorem}{Theorem}
\newtheorem{lemma}[theorem]{Lemma}
\usepackage{algorithm}
\usepackage{algorithmic}
\usepackage{graphicx}
\usepackage{booktabs}
\usepackage{multirow}
\usepackage{subcaption}
\usepackage{xcolor}

\newcommand{\RR}{\mathbb{R}}
\newcommand{\EE}{\mathbb{E}}

\begin{document}

\title{TT-FSI: Scalable Faithful Shapley Interactions via Tensor-Train}

\author{
  Ungsik Kim\\
  Gyeongsang National University\\
  \texttt{blpeng@gnu.ac.kr}
  \And
  Suwon Lee\thanks{(Corresponding author)}\\
  Gyeongsang National University\\
  leesuwon@gnu.ac.kr
}
\date{}

\maketitle

\begin{abstract}
The Faithful Shapley Interaction (FSI) index uniquely satisfies the faithfulness axiom among Shapley interaction indices, but computing FSI requires $O(d^\ell \cdot 2^d)$ time and existing implementations use $O(4^d)$ memory. We present \textbf{TT-FSI}, which exploits FSI's algebraic structure via Matrix Product Operators (MPO). Our main theoretical contribution is proving that the linear operator $v \mapsto \text{FSI}(v)$ admits an MPO representation with TT-rank $O(\ell d)$, enabling an efficient sweep algorithm with $O(\ell^2 d^3 \cdot 2^d)$ time and $O(\ell d^2)$ core storage---an exponential improvement over existing methods. Experiments on six datasets ($d=8$ to $d=20$) demonstrate up to 280$\times$ speedup over baseline, 85$\times$ over SHAP-IQ, and 290$\times$ memory reduction. TT-FSI scales to $d=20$ (1M coalitions) where all competing methods fail.
\end{abstract}

\section{Introduction}
\label{sec:intro}

Shapley values have become the de facto standard for feature attribution in explainable AI~\cite{arrieta2020explainable}, yet they are fundamentally limited to \emph{additive} explanations: each feature receives an independent contribution score, and these scores sum to the model output. This additive structure cannot capture \emph{interaction effects}---cases where features jointly contribute more (or less) than the sum of their individual effects. In tree ensembles and neural networks, such interactions often encode the most decision-relevant patterns, yet they remain invisible to standard SHAP analysis.

Shapley interaction indices address this gap by quantifying joint contributions of feature subsets~\cite{owen1972multilinear,grabisch1999axiomatic,fujimoto2006axiomatic}. Among proposed indices, the \textbf{Faithful Shapley Interaction (FSI)}~\cite{tsai2023faith} is uniquely characterized by the \emph{faithfulness axiom}: FSI scores induce the optimal $\ell$-order surrogate that minimizes Shapley-weighted approximation error. This optimality guarantee is absent in alternatives such as STII~\cite{sundararajan2020shapley} and SII~\cite{grabisch1999axiomatic}, making FSI the principled choice when explanation fidelity matters.

Despite its theoretical appeal, computing FSI is computationally challenging. The closed-form solution involves evaluating all $2^d$ coalitions and performing operations that scale as $O(d^\ell \cdot 2^d)$ for fixed interaction order $\ell$, or $O(3^d)$ when $\ell = d$. For moderate-dimensional data ($d \geq 15$), this becomes prohibitively expensive. Existing implementations such as SHAP-IQ~\cite{muschalik2024shapiq} in exact mode call \texttt{np.diag()} on a $2^d$-length vector, materializing a dense $2^d \times 2^d$ matrix and causing $O(4^d)$ memory consumption.

\paragraph{Our Approach.} Recent work IT-SHAP~\cite{hasegawa2025interaction} demonstrated that the Shapley-Taylor Interaction Index (STII) admits efficient computation via Tensor-Train (TT) decomposition, exploiting STII's discrete derivative structure. However, FSI---defined as the solution to a weighted least squares optimization problem (Eq.~\ref{eq:fsi_wls})---appears structurally incompatible with such methods. While STII admits local prefix-based transitions yielding TT-rank $O(d)$, FSI seems to require global regression over all coalitions.

Our key insight is that FSI's closed-form solution (Theorem 19 of~\cite{tsai2023faith}) can be decomposed into two operations---a M\"obius transform and a correction term---both of which admit efficient TT representations. The M\"obius transform has Kronecker product structure, giving it TT-rank 1. The correction term, despite its complex combinatorial weighting, depends only on subset \emph{cardinalities} $|S|$ and $|T|$, not on which specific features are included. This cardinality-dependence means we can compute it using an operator that tracks only two counters, yielding TT-rank $O(\ell d)$.

\paragraph{Contributions.} We present \textbf{TT-FSI}, a novel algorithm exploiting this algebraic structure. Our contributions build upon each other:
\begin{enumerate}
    \item \textbf{Theoretical foundation:} We prove that the linear operator $v \mapsto \text{FSI}(v)$ admits a Matrix Product Operator (MPO) representation with TT-rank $O(\ell d)$. This result---exploiting the TT structure of the \emph{operator itself}, not of the input $v$---enables the following algorithmic contribution.

    \item \textbf{Efficient algorithm:} Building on the MPO representation, we develop a left-to-right sweep algorithm achieving time complexity $O(\ell^2 d^3 \cdot 2^d)$ and core storage $O(\ell d^2)$---an exponential memory improvement over existing $O(4^d)$ methods.

    \item \textbf{Empirical validation:} We demonstrate that this theoretical and algorithmic framework translates to practical gains: up to 280$\times$ speedup over baseline, 85$\times$ over SHAP-IQ, and 290$\times$ memory reduction on six real-world datasets ($d=8$ to $d=20$). TT-FSI computes FSI for $d=20$ (over 1 million coalitions) in under 10 seconds, where all competing methods fail.
\end{enumerate}

\section{Background}
\label{sec:background}
This section introduces the key concepts needed to understand TT-FSI. We first review Shapley values (\S\ref{sec:background}.1), then discuss interaction indices with focus on FSI (\S\ref{sec:background}.2), and finally introduce the Tensor-Train decomposition that enables our efficient algorithm (\S\ref{sec:background}.3).

\subsection{Cooperative Game Theory and Shapley Values}

Let $[d] = \{1, 2, \ldots, d\}$ be a set of $d$ players (features). A \emph{cooperative game} is defined by a value function $v: 2^{[d]} \to \RR$ that assigns a worth to each coalition $S \subseteq [d]$, with $v(\emptyset) = 0$. In the context of machine learning, $v(S)$ typically represents the model's prediction when only features in $S$ are present (with features in $[d] \setminus S$ marginalized out).

The \textbf{Shapley value}~\cite{shapley1953value} assigns to each player $i$ a fair share of the total worth:
\begin{equation}
\phi_i(v) = \sum_{S \subseteq [d] \setminus \{i\}} \frac{|S|!(d-|S|-1)!}{d!} [v(S \cup \{i\}) - v(S)].
\end{equation}

The Shapley value satisfies several desirable axioms including efficiency ($\sum_i \phi_i = v([d]) - v(\emptyset)$), symmetry, linearity, and null player properties.

\subsection{Shapley Interaction Indices}

Various interaction indices extend Shapley values to feature subsets, including SII~\cite{grabisch1999axiomatic}, k-SII, STII~\cite{sundararajan2020shapley}, and BII (see Appendix~\ref{app:indices} for definitions). Among these, the \textbf{Faithful Shapley Interaction (FSI)}~\cite{tsai2023faith} is uniquely characterized by the \emph{faithfulness axiom}: the interaction scores induce the optimal $\ell$-order polynomial surrogate under Shapley-weighted regression. FSI is defined as:
\begin{equation}
\label{eq:fsi_wls}
\min_{\{I_S\}_{|S| \leq \ell}} \sum_{T \subseteq [d]} w_T \left( v(T) - \sum_{S \subseteq T, |S| \leq \ell} I_S \right)^2,
\end{equation}
with specific weights $w_T$ derived from Shapley's axioms. The resulting surrogate $g(T) = \sum_{S \subseteq T, |S| \leq \ell} I_S$ is the unique \emph{faithful} approximation: for $\ell$-additive games (where $v$ has M\"obius support on $|S| \leq \ell$), FSI recovers $v$ exactly.

\paragraph{Closed-Form Solution.} While the optimization problem~\eqref{eq:fsi_wls} appears to require solving a weighted least squares system over $2^d$ coalitions, the special structure of Shapley weights admits a remarkably compact solution. Tsai et al.~\cite{tsai2023faith} proved that FSI has a closed-form (their Theorem 19):
\begin{equation}
\label{eq:fsi_closed}
E^{\text{FSI}}_S(v, \ell) = a(v, S) + C_S(v, \ell),
\end{equation}
where $a(v, S) = \sum_{T \subseteq S} (-1)^{|S|-|T|} v(T)$ is the \textbf{M\"obius transform} and $C_S(v, \ell)$ is a \textbf{correction term}:
\begin{equation}
\label{eq:correction}
C_S(v, \ell) = (-1)^{\ell-|S|} \frac{|S|}{\ell+|S|} \binom{\ell}{|S|} \sum_{\substack{T \supset S \\ |T| > \ell}} \frac{\binom{|T|-1}{\ell}}{\binom{|T|+\ell-1}{\ell+|S|}} a(v, T).
\end{equation}

This correction term adjusts the M\"obius transform to ensure faithfulness, and it only involves coalitions $T$ with $|T| > \ell$.

\subsection{Tensor-Train Decomposition}

The \textbf{Tensor-Train (TT) decomposition}~\cite{oseledets2011tensor,kolda2009tensor} represents a $d$-dimensional tensor $\mathcal{A} \in \RR^{n_1 \times \cdots \times n_d}$ as a product of 3-way tensors (cores):
\begin{equation}
\mathcal{A}_{i_1, \ldots, i_d} = \sum_{\alpha_1, \ldots, \alpha_{d-1}} G^{(1)}_{i_1, \alpha_1} G^{(2)}_{\alpha_1, i_2, \alpha_2} \cdots G^{(d)}_{\alpha_{d-1}, i_d}.
\end{equation}

The \textbf{TT-rank} is the tuple $(r_1, \ldots, r_{d-1})$ where $r_k$ is the dimension of the bond index $\alpha_k$. A key property is that tensors with low TT-rank can be stored and manipulated efficiently in $O(dnr^2)$ space and time, where $r = \max_k r_k$ and $n = \max_k n_k$.

\paragraph{Matrix Product Operators (MPO).} For our purposes, we work with \emph{Matrix Product Operators}~\cite{schollwock2011density}, which represent linear maps $M: \RR^{2^d} \to \RR^{2^d}$ in TT format:
\begin{equation}
M_{\sigma, \tau} = \sum_{\alpha_1, \ldots, \alpha_{d-1}} G^{(1)}_{\sigma_1, \tau_1, \alpha_1} G^{(2)}_{\alpha_1, \sigma_2, \tau_2, \alpha_2} \cdots G^{(d)}_{\alpha_{d-1}, \sigma_d, \tau_d},
\end{equation}
where $\sigma = (\sigma_1, \ldots, \sigma_d)$ and $\tau = (\tau_1, \ldots, \tau_d)$ are binary indices representing subsets.

\paragraph{Why TT Matters for FSI.} A central advantage of TT decomposition is that \emph{operators with specific algebraic structures often admit low TT-rank representations}. If an operator $M$ can be written as a Kronecker product $M = \bigotimes_{i=1}^d M_i$ of small matrices, it has TT-rank 1 (Lemma~\ref{lem:kronecker_rank1} in Appendix~\ref{app:proofs}). More generally, operators that ``locally'' process each dimension with bounded state tracking can be represented with polynomial TT-rank.

The M\"obius operator $\mu_\downarrow$ on the Boolean lattice---central to FSI computation---has precisely such a Kronecker structure: $\mu_\downarrow = \bigotimes_{i=1}^{d} M_i$ where $M_i = \begin{psmallmatrix} 1 & 0 \\ -1 & 1 \end{psmallmatrix}$. This implies $\mu_\downarrow$ has \textbf{TT-rank 1} and can be applied in $O(d \cdot 2^d)$ time (see Appendix~\ref{app:incidence}).

Recent work IT-SHAP~\cite{hasegawa2025interaction} exploited this structure for the Shapley-Taylor Interaction Index (STII), which has a single-counter structure yielding TT-rank $O(d)$. However, FSI's correction term (Eq.~\ref{eq:correction}) requires tracking \emph{two} counters---$|S|$ and $|T|$---increasing TT-rank to $O(\ell d)$. Despite this added complexity, the cardinality-dependence still enables polynomial-rank representation, as we formalize in \S\ref{sec:method}.

\section{Related Work}
\label{sec:related}

Before presenting our algorithm, we situate TT-FSI within the context of landscape of explainability methods and tensor decomposition techniques.

\paragraph{Shapley-based Explanations.} LIME~\cite{ribeiro2016should} pioneered local surrogate-based explanations, while earlier work by \v{S}trumbelj and Kononenko~\cite{vstrumbelj2014explaining} developed Shapley-based feature contributions. SHAP~\cite{lundberg2017unified} unified several feature attribution methods under the Shapley value framework, becoming central to modern XAI~\cite{guidotti2018survey}. Covert et al.~\cite{covert2021explaining} provide a comprehensive taxonomy of removal-based explanations unifying 26 methods. Integrated Gradients~\cite{sundararajan2017axiomatic} provides an axiomatic approach for neural networks. TreeSHAP~\cite{lundberg2020local} enables polynomial-time exact computation for tree models by exploiting tree structure. FastSHAP~\cite{jethani2022fastshap} trains an auxiliary model to predict Shapley values in real-time, trading exactness for speed.

\paragraph{Interaction Indices.} Beyond Shapley-based methods, statistical approaches such as the H-statistic~\cite{friedman2008predictive} and functional ANOVA decomposition~\cite{hooker2007generalized} measure feature interactions through variance partitioning. In game-theoretic approaches, the Shapley Interaction Index (SII)~\cite{grabisch1999axiomatic} and Shapley-Taylor Interaction Index (STII)~\cite{sundararajan2020shapley} extend Shapley values to interactions but lack faithfulness. Faith-Shap (FSI)~\cite{tsai2023faith} is the unique faithful interaction index.

\paragraph{Efficient Computation.} Computing Shapley values is \#P-hard in general~\cite{van2022tractability,arenas2023complexity}, motivating both exact algorithms for tractable model classes and sampling-based approximations. The choice between interventional and observational conditionals~\cite{janzing2020feature} and between model-faithful and data-faithful explanations~\cite{chen2020true} further complicates evaluation, while adversarial vulnerabilities of post-hoc methods~\cite{slack2020fooling} underscore the need for principled approaches. The shapiq library~\cite{muschalik2024shapiq,muschalik2024shapiq} provides a unified framework for computing various interaction indices, including both exact and approximate methods. Its exact mode (SHAP-IQ) scales poorly with dimension, creating dense $2^d \times 2^d$ matrices via \texttt{np.diag()} and causing $O(4^d)$ memory consumption. Sampling-based alternatives~\cite{castro2009polynomial,mitchell2022sampling} such as SVARM-IQ~\cite{kolpaczki2024svarm} use stratified sampling to estimate interaction indices with probabilistic guarantees, achieving polynomial sample complexity for fixed approximation error. Extensions handle dependent features~\cite{aas2021explaining} and global importance aggregation~\cite{covert2020understanding}. These sampling approaches excel when approximate values suffice, but exact computation remains necessary for rigorous analysis, debugging model behavior, and benchmarking approximation methods---precisely the use cases TT-FSI addresses.

\paragraph{Tensor Networks.} Tensor-Train decomposition~\cite{oseledets2011tensor,cichocki2016tensor} and Matrix Product States~\cite{schollwock2011density} originated in numerical linear algebra and quantum physics. Connections between tensor networks and graphical models~\cite{robeva2019duality} provide additional theoretical grounding. Recent work has applied tensor networks to machine learning~\cite{stoudenmire2016supervised,novikov2015tensorizing}. TT-FSI makes no assumption on $v$; we exploit the TT structure of the \emph{FSI operator} rather than the input.

\paragraph{Relation to IT-SHAP.} IT-SHAP~\cite{hasegawa2025interaction} applied TT decomposition to the Shapley-Taylor Interaction Index (STII), achieving substantial speedups. STII's success stems from its combinatorial simplicity: discrete derivatives yield weights depending on a \emph{single} cardinality, representable by a single-counter finite state machine with TT-rank $O(d)$. FSI, by contrast, is defined implicitly via weighted least squares (Eq.~\ref{eq:fsi_wls})---an optimization problem rather than a direct formula---making it appear structurally incompatible with tensor methods. Indeed, IT-SHAP does not mention FSI, leaving the ``least squares class'' of interaction indices outside tensor-based approaches.

We show that this incompatibility is only superficial. By analyzing Tsai et al.'s closed-form solution~\cite{tsai2023faith}, we discover that FSI decomposes into incidence algebra operations: a rank-1 M\"obius transform plus a correction term whose weights depend on \emph{two} cardinalities ($|S|$ and $|T|$). This two-counter structure admits a finite state machine with $O(\ell d)$ states---still tractable. TT-FSI thus extends tensor methods beyond combinatorially-defined indices to the optimization-defined FSI, bridging a gap that prior work left unaddressed.

\section{Method: TT-FSI}
\label{sec:method}

\subsection{Overview}

FSI's closed-form (Eq.~\ref{eq:fsi_closed}) naturally decomposes into two operations: the \textbf{M\"obius transform} $a(v, S)$, which extracts the ``pure'' interaction contribution of subset $S$ by removing redundant contributions from smaller coalitions, and a \textbf{correction term} $C_S(v, \ell)$, which redistributes contributions from large coalitions ($|T| > \ell$) to ensure faithfulness.

Both operations are \emph{linear} in $v$. The M\"obius transform $a(v, S) = \sum_{T \subseteq S} (-1)^{|S|-|T|} v(T)$ is a linear combination of value function entries, computed by the operator $\mu_\downarrow$. The correction term (Eq.~\ref{eq:correction}) sums over $a(v, T)$ for $|T| > \ell$---itself linear in $v$---with coefficients depending only on cardinalities $|S|$ and $|T|$, computed by the operator $A_{\mathrm{trunc}}$. We thus view FSI as a linear operator mapping $v \in \RR^{2^d}$ to interaction scores:
\begin{equation}
\label{eq:fsi_decomp}
E^{\text{FSI}}(v, \ell) = \underbrace{\mu_\downarrow \cdot v}_{\text{M\"obius transform}} + \underbrace{A_{\text{trunc}} \cdot (\mu_\downarrow \cdot v)}_{\text{Correction term}},
\end{equation}
where $\mu_\downarrow$ is the M\"obius operator and $A_{\text{trunc}}$ is a correction MPO. Our key insight is that \emph{both operations admit efficient TT representations}: the M\"obius transform has TT-rank 1 due to its Kronecker structure (\S\ref{sec:mobius}), and the correction term has TT-rank $O(\ell d)$ because it only needs to track two bounded counters (\S\ref{sec:correction}).

\paragraph{Roadmap.} We proceed as follows: (a) establish the TT structure of each component---Möbius transform (\S\ref{sec:mobius}) and correction term (\S\ref{sec:correction}); (b) develop an efficient sweep algorithm to apply the combined MPO (\S\ref{sec:contraction}); (c) introduce a precontraction optimization (\S\ref{sec:precontract}); and (d) analyze the resulting complexity (\S\ref{sec:complexity}).

\subsection{M\"obius Transform as Rank-1 MPO}
\label{sec:mobius}

The M\"obius transform $a(v, S) = \sum_{T \subseteq S} (-1)^{|S|-|T|} v(T)$ can be computed by applying the operator $\mu_\downarrow$ to the value function vector $v$. Due to the Kronecker structure, each core is:
\begin{equation}
G^{(k)}_{\text{M\"ob}}[\sigma_k, \tau_k] = M[\sigma_k, \tau_k] = \begin{cases} 1 & \sigma_k = 0, \tau_k = 0 \\ 1 & \sigma_k = 1, \tau_k = 1 \\ -1 & \sigma_k = 1, \tau_k = 0 \\ 0 & \text{otherwise} \end{cases}
\end{equation}

Since all cores have bond dimension 1, the M\"obius transform has \textbf{TT-rank 1} and can be computed in $O(d \cdot 2^d)$ time.

\subsection{Correction Term as Polynomial-Rank MPO}
\label{sec:correction}

The correction term in Equation~\eqref{eq:correction} involves summing over supersets $T \supset S$ with $|T| > \ell$, weighting by coefficients that depend on $|S|$ and $|T|$, and filtering to outputs with $|S| \leq \ell$.

\paragraph{Intuition: Why Polynomial Rank?} The correction term can be computed by an operator that processes each feature dimension sequentially, tracking only \emph{two counters}: the running size of the output set $|S|$ and the running size of the input set $|T|$. Since the weights depend only on these cardinalities (not on which specific features are included), we can model this computation as a \emph{finite-state machine} with state space $(|S|, |T|)$. At each position $k$, the number of possible states is bounded by $O(\ell k)$ because $|S| \leq \ell$. This finite-state structure naturally maps to an MPO with polynomial TT-rank---far smaller than the exponential $2^d$ that would be required by a naive approach. Figure~\ref{fig:fsm} illustrates this structure for $d=3$ and $\ell=2$.

We now formalize this intuition. The key insight is that the abstract state $(|S|, |T|)$ can be tracked incrementally: at each position $k$, we maintain \emph{running counts} $(s_k, t_k)$ representing the partial sums accumulated so far, which equal $(|S|, |T|)$ at the final position $k = d$.

\paragraph{Formal Construction.} We construct an MPO $A_{\text{trunc}}$ that realizes the finite-state machine described above. The MPO's bond indices encode the running state $(s_k, t_k)$, where $s_k = \sum_{i=1}^k \sigma_i$ is the running count of output bits (size of $S$) and $t_k = \sum_{i=1}^k \tau_i$ is the running count of input bits (size of $T$).

\paragraph{State Space.} At position $k$, the state space is:
\begin{equation}
\{(s, t) : 0 \leq s \leq \min(k, \ell), \quad 0 \leq t \leq k, \quad t \geq s\}
\end{equation}

The constraint $s \leq \min(k, \ell)$ implements the truncation to interaction order $\ell$. The constraint $t \geq s$ is a \emph{derived necessary condition}: we enforce the per-bit constraint $\tau_i \geq \sigma_i$ in local transitions (see below), which implies $t_k \geq s_k$ automatically. The state constraint helps bound the number of reachable states.

\paragraph{Bond Dimension.} Note that while the constraint $t \geq s$ reduces the exact count of valid states, the number of states is \emph{upper bounded by}:
\begin{equation}
D_k \leq (\min(k, \ell) + 1)(k + 1) = O(\ell k),
\end{equation}
which suffices for our TT-rank analysis. The maximum bond dimension is thus $D_{\max} = O(\ell d)$.

\begin{figure}[t]
    \centering
    \includegraphics[width=0.85\columnwidth,height=0.38\textheight,keepaspectratio=false]{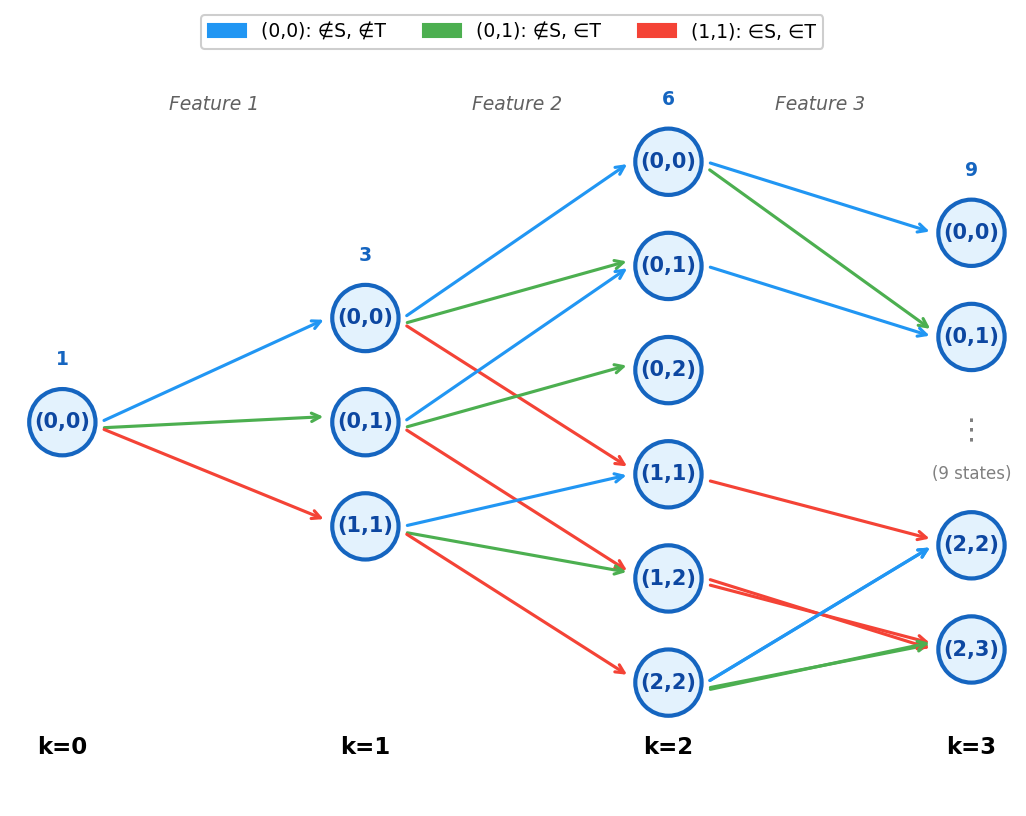}
    \caption{Finite-state machine for correction operator ($d=3$, $\ell=2$). States $(s,t)$ track cumulative counts. Colors: blue ($\notin S, \notin T$), green ($\notin S, \in T$), red ($\in S \cap T$). State count grows polynomially: $1 \to 3 \to 6 \to 9$.}
    \label{fig:fsm}
\end{figure}

\begin{lemma}[TT-rank upper bound]
\label{lem:ttrank_state}
The TT-rank of $A_{\mathrm{trunc}}$ at bond $k$ is upper bounded by the number of reachable states, i.e., $O(\ell k)$. The maximum TT-rank is thus $O(\ell d)$. (Proof in Appendix~\ref{app:proofs}.)
\end{lemma}

\paragraph{Core Construction.} For internal sites $k = 1, \ldots, d-1$, the core implements pure state transitions:
\begin{equation}
G^{(k)}[\alpha_{\text{prev}}, \sigma, \tau, \alpha_{\text{next}}] = \begin{cases}
1 & \text{if valid transition} \\
0 & \text{otherwise}
\end{cases}
\end{equation}

A transition is valid if: (i)~$\tau \geq \sigma$ (per-bit superset enforcement: if $\sigma_k = 1$ then $\tau_k = 1$); (ii)~$s_{\text{next}} = s_{\text{prev}} + \sigma$ and $t_{\text{next}} = t_{\text{prev}} + \tau$ (counter updates); and (iii)~$s_{\text{next}} \leq \min(k, \ell)$ (output truncation).

For the final site $k = d$, the core applies the weight function:
\begin{equation}
G^{(d)}[\alpha_{\text{prev}}, \sigma, \tau, 0] = \begin{cases}
w(s, t) & \text{if valid and } s \leq \ell, t > \ell \\
0 & \text{otherwise}
\end{cases}
\end{equation}

where the weight $w(s, t)$ extracts the coefficient from the correction formula (Eq.~\ref{eq:correction}), substituting $s = |S|$ and $t = |T|$ (correctness proof in Lemma~\ref{lem:weight_correct}):
\begin{equation}
\label{eq:weight}
w(s, t) = (-1)^{\ell - s} \cdot \frac{s}{\ell + s} \cdot \binom{\ell}{s} \cdot \frac{\binom{t-1}{\ell}}{\binom{t+\ell-1}{\ell+s}}.
\end{equation}

\paragraph{Numerical Stability.} To avoid overflow in binomial coefficients for large $d$, we compute weights using log-gamma functions. The complete log-stable implementation is provided in Algorithm~2 of the Supplementary Material.

The preceding construction yields an MPO with polynomial TT-rank $O(\ell d)$, but naive application to a value function vector would still require exponential time. The next section addresses this computational challenge.

\subsection{Efficient MPO Contraction}
\label{sec:contraction}

The polynomial TT-rank $O(\ell d)$ established in \S\ref{sec:correction} is the key enabler of efficient computation. A naive implementation of MPO-vector multiplication iterates over all $2^d \times 2^d$ input-output pairs, resulting in $O(4^d \cdot D)$ complexity---worse than brute-force FSI. However, the TT structure allows us to avoid this exponential blowup: we employ a \textbf{left-to-right sweep} that processes one dimension at a time, maintaining an intermediate tensor that contracts the MPO core with the input progressively. Algorithm~\ref{alg:mpo_sweep} formalizes this approach.

\begin{algorithm}[t]
\caption{Efficient MPO-Vector Contraction (Left-to-Right Sweep)}
\label{alg:mpo_sweep}
\begin{algorithmic}[1]
\REQUIRE MPO cores $\{G^{(k)}\}_{k=1}^d$, input vector $v \in \RR^{2^d}$
\ENSURE Output vector $u = M \cdot v$
\STATE $X \gets v$ reshaped to $(1, 1, 2^d)$ \COMMENT{(out-prefix, bond, in-suffix)}
\FOR{$k = 1$ to $d$}
    \STATE \COMMENT{$X \in \RR^{2^{k-1} \times D_{k-1} \times 2^{d-k+1}}$}
    \STATE Reshape $X$ to $(2^{k-1} \cdot D_{k-1}, 2, 2^{d-k})$ \COMMENT{merge, split off $\tau_k$}
    \STATE Contract: $X'[m, \sigma, \beta, r] \gets \sum_{\tau, \alpha} X[m', \tau, r] \cdot G^{(k)}[\alpha, \sigma, \tau, \beta]$
    \STATE Reshape $X'$ to $(2^k, D_k, 2^{d-k})$ \COMMENT{expand out-prefix with $\sigma_k$}
    \STATE $X \gets X'$
\ENDFOR
\STATE Reshape $X$ to $u \in \RR^{2^d}$
\RETURN $u$
\end{algorithmic}
\end{algorithm}

\paragraph{Intermediate Tensor Structure.} At each step $k$, the intermediate tensor has shape $(2^{k-1}, D_{k-1}, 2^{d-k+1})$, where the three dimensions represent partial output configurations $(\sigma_1, \ldots, \sigma_{k-1})$, MPO bond states encoding $(s_{k-1}, t_{k-1})$ counters, and remaining input configurations $(\tau_k, \ldots, \tau_d)$, respectively.

\noindent\textbf{Remark.} While the output dimension appears as $2^d$, the correction MPO produces non-zero outputs \emph{only} for $|S| \leq \ell$ due to the truncation constraint in the state space. The $O(2^d)$ intermediate size is unavoidable because we must read the full input $v \in \RR^{2^d}$. Optionally, one can compress the output-prefix dimension by storing only prefixes with $s_k \leq \ell$, reducing $2^k$ to $\sum_{s \leq \ell} \binom{k}{s}$; we use the dense representation for simplicity since input I/O dominates memory.

\noindent\textbf{Sparse Core Storage.} The MPO cores are highly sparse. For each input state $\alpha_{\text{prev}}$, there are at most 3 valid transitions corresponding to $(\sigma, \tau) \in \{(0,0), (0,1), (1,1)\}$ (since $\tau \geq \sigma$). Thus each core $G^{(k)}$ has $O(D_{k-1})$ nonzero entries, and total storage is $O(\ell d^2)$ (Lemma~\ref{lem:sparse_storage}), which is even tighter than the dense bound $O(\ell^2 d^3)$.

Each contraction step costs $O(2^k \cdot D_{k-1} \cdot D_k \cdot 2^{d-k})$. Summing over all $k$ yields $O(\ell^2 d^3 \cdot 2^d)$ (detailed derivation in Lemma~\ref{lem:time_complexity}).
This bound is conservative. Since transitions are \emph{deterministic} (each $(\alpha_{\text{prev}}, \sigma, \tau)$ maps to at most one $\alpha_{\text{next}}$), exploiting sparsity can reduce the per-step cost to depend on $\text{nnz}(G^{(k)}) = O(D_{k-1})$ rather than $D_{k-1} \cdot D_k$, explaining why empirical runtimes often beat the theoretical upper bound.

\subsection{Precontraction Optimization}
\label{sec:precontract}

Since the correction term applies $A_{\text{trunc}}$ to the M\"obius-transformed input $a = \mu_\downarrow v$, we can precompute the composition:
\begin{equation}
B_{\text{corr}} := A_{\text{trunc}} \circ \mu_\downarrow.
\end{equation}

This allows computing the correction $C = B_{\text{corr}} \cdot v$ in a single pass. The final FSI is then $\text{FSI} = \mu_\downarrow v + B_{\text{corr}} v$. Core-level composition sums over the intermediate index $\tau$:
\begin{equation}
B_{\text{corr}}^{(k)}[\alpha, \sigma, \rho, \beta] = \sum_{\tau} A^{(k)}[\alpha, \sigma, \tau, \beta] \cdot M[\tau, \rho],
\end{equation}
where $M[\tau, \rho] = \delta_{\tau,\rho} - \delta_{\tau,1}\delta_{\rho,0}$ is the M\"obius matrix. Since $M$ is $2 \times 2$, this simplifies to two cases: $B^{(k)}[\alpha, \sigma, 0, \beta] = A^{(k)}[\alpha, \sigma, 0, \beta] - A^{(k)}[\alpha, \sigma, 1, \beta]$ and $B^{(k)}[\alpha, \sigma, 1, \beta] = A^{(k)}[\alpha, \sigma, 1, \beta]$.

Since $\mu_\downarrow$ has TT-rank 1, precontraction does not increase the bond dimension of $B_{\text{corr}}$ beyond that of $A_{\text{trunc}}$ (Lemma~\ref{lem:precontract_rank}).

\subsection{Complexity Analysis}
\label{sec:complexity}

\begin{table}[t]
\caption{Complexity comparison of FSI computation methods. ``Core Space'' denotes storage for algorithm-specific data structures, excluding the $O(2^d)$ input/output vectors common to all methods.}
\label{tab:complexity}
\centering
\small
\begin{tabular}{lccc}
\toprule
\textbf{Method} & \textbf{Time} & \textbf{Core Space} & \textbf{Space Complexity} \\
\midrule
Baseline~\cite{tsai2023faith} & $O(d^\ell \cdot 2^d)$ & $O(1)$ & $O(2^d)$ \\
SHAP-IQ~\cite{muschalik2024shapiq} & $O(d^{2\ell} \cdot 2^d)$ & $O(4^d)$ & $O(4^d)$ \\
TT-FSI (ours) & $O(\ell^2 d^3 \cdot 2^d)$ & $O(\ell d^2)$ & $O(\ell d \cdot 2^d)$ \\
\bottomrule
\end{tabular}
\end{table}

Table~\ref{tab:complexity} summarizes the complexity of different approaches.

\paragraph{Time Complexity.} The baseline FSI computation iterates over all $(S, T)$ pairs where $T \supseteq S$ and $|S| \leq \ell$. The number of such pairs is $\sum_{t=0}^{\ell} \binom{d}{t} \cdot 2^{d-t}$, which by the binomial theorem equals $3^d$ when $\ell = d$, but simplifies to $O(d^\ell \cdot 2^d)$ for fixed $\ell \ll d$. SHAP-IQ's \texttt{compute\_fii()} follows a similar iteration pattern, plus an $O(d^{2\ell} \cdot 2^d)$ term from the least-squares solver operating on a $(2^d \times O(d^\ell))$ matrix.

TT-FSI's time complexity $O(\ell^2 d^3 \cdot 2^d)$ has a \emph{fixed polynomial} dependence on $d$, independent of $\ell$'s relationship to $d$. This yields speedup ratio $(d^\ell \cdot 2^d) / (\ell^2 d^3 \cdot 2^d) = d^{\ell-3}/\ell^2$, which grows polynomially for $\ell \geq 4$ and is $O(1)$ for $\ell \leq 3$. Our experiments at $\ell \in \{2, 3\}$ show speedups from constant-factor improvements in the TT structure.

\paragraph{Memory Complexity.} More critically, SHAP-IQ's \texttt{compute\_fii()} creates a dense diagonal weight matrix via \texttt{np.diag()}, allocating a full $2^d \times 2^d$ array---hence $O(4^d)$ memory. This is the root cause of SHAP-IQ's out-of-memory failures at $d \geq 16$. TT-FSI's core space complexity is \textbf{polynomial} $O(\ell d^2)$, with peak memory $O(\ell d \cdot 2^d)$ dominated by the intermediate tensor during contraction. All methods require $O(2^d)$ space for input $v$ and output FSI scores.

\paragraph{NC$^2$ Complexity.} TT-FSI belongs to NC$^2$~\cite{arora2009computational}: the computation can be performed by circuits of depth $O(\log n \cdot \log\log n)$ where $n=2^d$. This suggests potential for further speedup via GPU parallelization in future work (proof in Appendix~\ref{app:proofs}).

\section{Experiments}
\label{sec:experiments}

We evaluate TT-FSI on six datasets: California Housing~\cite{pace1997sparse} ($d=8$), Diabetes~\cite{efron2004least} ($d=10$), COMPAS~\cite{angwin2022machine} ($d=11$), Adult~\cite{asuncion2007uci} ($d=14$), Bank Marketing~\cite{asuncion2007uci} ($d=16$), and German Credit~\cite{asuncion2007uci} ($d=20$), training four model types (DT, XGBoost, LGBM, MLP). We compare against baseline FSI and SHAP-IQ~\cite{muschalik2024shapiq} v0.1.1 exact mode; full details in Appendix~\ref{app:experiments}.

\paragraph{Accuracy.} TT-FSI produces numerically identical results: max difference vs.\ baseline is $8.0 \times 10^{-15}$ (machine precision), vs.\ SHAP-IQ is $1.8 \times 10^{-9}$.

\subsection{Runtime Performance}

\begin{table}[t]
\caption{Runtime summary at $\ell=3$ (mean over 4 models $\times$ 3 seeds $\times$ 3 runs). Full breakdown in Table~\ref{tab:runtime_full}.}
\label{tab:runtime}
\centering
\begin{tabular}{lc|rrr|rr}
\toprule
Dataset & $d$ & TT-FSI & SHAP-IQ & Baseline & \multicolumn{2}{c}{Speedup} \\
& & & & & vs SHAP-IQ & vs Baseline \\
\midrule
California & 8 & \textbf{0.8ms} & 2.5ms & 21.7ms & 3$\times$ & 27$\times$ \\
Diabetes & 10 & \textbf{1.5ms} & 18.5ms & 162.7ms & 12$\times$ & 105$\times$ \\
COMPAS & 11 & \textbf{2.6ms} & 89ms & 426ms & 34$\times$ & 164$\times$ \\
Adult & 14 & \textbf{25ms} & 1997ms & 6909ms & 80$\times$ & 276$\times$ \\
Bank & 16 & \textbf{156ms} & OOM$^*$ & $>$1h$^\dagger$ & -- & -- \\
German & 20 & \textbf{7.95s} & OOM$^*$ & $>$1h$^\dagger$ & -- & -- \\
\bottomrule
\multicolumn{7}{l}{\scriptsize $^*$Out of memory. $^\dagger$Exceeds time limit ($O(4^d)$ iterations).}
\end{tabular}
\end{table}

Table~\ref{tab:runtime} summarizes runtime performance (full per-model breakdown in Appendix~\ref{app:experiments}).

\paragraph{Speedup Analysis.} TT-FSI outperforms baseline FSI by 20--280$\times$ across all datasets, with speedup increasing with $d$: from $\sim$27$\times$ at $d=8$ to $\sim$276$\times$ at $d=14$. This aligns with our complexity analysis---baseline's $O(d^\ell \cdot 2^d)$ scaling incurs a polynomial factor that TT-FSI avoids. Beyond $d=14$, baseline becomes impractical due to $O(4^d)$ nested iterations, while SHAP-IQ exhausts memory ($O(4^d)$ space). TT-FSI scales gracefully with $\ell$, and its efficiency is model-agnostic---runtime depends only on $d$ and $\ell$.

\paragraph{Comparison with Approximations.} TT-FSI computes \emph{exact} FSI faster than sampling-based approximators. At $d=14$, $\ell=3$, TT-FSI (41ms) is 8$\times$ faster than SHAP-IQ Monte Carlo (337ms) and 377$\times$ faster than SVARM-IQ~\cite{kolpaczki2024svarm} (15.5s), while approximators incur RMSE of 0.15--0.40 (Table~\ref{tab:svarmiq} in Appendix).

\paragraph{Memory and Scaling.} TT-FSI achieves up to \textbf{195$\times$ memory reduction} versus SHAP-IQ at $d=14$ (21MB vs 4.1GB), enabling computation at $d=20$ where SHAP-IQ requires $>$60GB. Detailed memory analysis in Appendix~\ref{app:memory}.

\section{Conclusion}
\label{sec:conclusion}
We presented TT-FSI, a memory-efficient algorithm for computing Faithful Shapley Interactions via Tensor-Train decomposition. By exploiting the algebraic structure of FSI's closed form, we achieve TT-rank $O(\ell d)$ with time complexity $O(\ell^2 d^3 \cdot 2^d)$ and sparse core storage $O(\ell d^2)$. In practice, this translates to up to 280$\times$ speedup over baseline, 85$\times$ over SHAP-IQ, and 290$\times$ memory reduction, enabling successful computation at $d=20$ where all existing methods fail due to memory exhaustion.

\paragraph{Limitations and Future Work.} TT-FSI requires the full $2^d$ value function, limiting applicability to $d \lesssim 25$---sufficient for tabular ML but excluding high-dimensional domains. Future work could combine TT-FSI with sampling-based value estimation, or integrate model-specific optimizations (e.g., TreeSHAP-style algorithms). Our NC$^2$ result (Appendix~\ref{app:proofs}) establishes high parallelizability; GPU experiments (Appendix~\ref{app:gpu}) confirm up to 19$\times$ speedup via Tensor Cores.

\bibliographystyle{plain}
\bibliography{references}

\newpage
\appendix

\section{Shapley Interaction Indices}
\label{app:indices}

\paragraph{Shapley Interaction Index (SII).} The SII~\cite{grabisch1999axiomatic} extends Shapley values via discrete derivatives:
\begin{equation}
\text{SII}_S(v) = \sum_{T \subseteq [d] \setminus S} \frac{|T|!(d-|S|-|T|)!}{(d-|S|+1)!} \Delta_S v(T),
\end{equation}
where $\Delta_S v(T) = \sum_{R \subseteq S} (-1)^{|S|-|R|} v(T \cup R)$. SII satisfies linearity and symmetry but \emph{not} efficiency for $|S| > 1$.

\paragraph{k-SII.} The k-Shapley Interaction Index~\cite{sundararajan2020shapley} aggregates higher-order SII terms to restore efficiency: $\sum_{|S| \leq \ell} I_S = v([d]) - v(\emptyset)$.

\paragraph{Shapley-Taylor Interaction Index (STII).} STII~\cite{sundararajan2020shapley} distributes higher-order M\"obius coefficients to top-order interactions:
\begin{equation}
\text{STII}_S(v, \ell) = \sum_{T \supseteq S} \frac{(d - |T|)! (|T| - \ell)!}{(d - \ell + 1)!} \Delta_S v(T \setminus S).
\end{equation}
STII satisfies efficiency but \emph{not} faithfulness.

\paragraph{Banzhaf Interaction Index (BII).} BII uses uniform weighting instead of Shapley weights, offering computational simplicity.

\section{Incidence Algebra Details}
\label{app:incidence}

The M\"obius function on Boolean lattices connects to pseudo-Boolean function approximation~\cite{hammer1992approximations}, where Shapley values and Banzhaf indices arise naturally as best linear approximations. The M\"obius function $\mu$ and zeta function $\zeta$ on the Boolean lattice~\cite{bjorklund2007fourier} are:
\begin{equation}
\mu(S, T) = \begin{cases} (-1)^{|S|-|T|} & T \subseteq S \\ 0 & \text{otherwise} \end{cases}, \quad
\zeta(S, T) = \begin{cases} 1 & T \subseteq S \\ 0 & \text{otherwise} \end{cases}
\end{equation}

These have Kronecker product structure: $\mu_\downarrow = \bigotimes_{i=1}^{d} M_i$ where $M_i = \begin{psmallmatrix} 1 & 0 \\ -1 & 1 \end{psmallmatrix}$, implying TT-rank 1.

\section{Proofs}
\label{app:proofs}

\begin{lemma}[Kronecker product has TT-rank 1~\cite{oseledets2011tensor}]
\label{lem:kronecker_rank1}
If a linear operator $M: \RR^{n^d} \to \RR^{n^d}$ has Kronecker product structure $M = \bigotimes_{i=1}^{d} M_i$ where each $M_i \in \RR^{n \times n}$, then $M$ has TT-rank 1.
\end{lemma}

\begin{proof}
The TT decomposition of $M$ is given by cores $G^{(k)} = M_k$ with trivial bond indices (dimension 1). Explicitly, for input index $\tau = (\tau_1, \ldots, \tau_d)$ and output index $\sigma = (\sigma_1, \ldots, \sigma_d)$:
\begin{equation}
M_{\sigma, \tau} = \prod_{k=1}^{d} (M_k)_{\sigma_k, \tau_k} = G^{(1)}_{\sigma_1, \tau_1} \cdot G^{(2)}_{\sigma_2, \tau_2} \cdots G^{(d)}_{\sigma_d, \tau_d}.
\end{equation}
Since all bond dimensions equal 1, the TT-rank is 1. The M\"obius operator $\mu_\downarrow = \bigotimes_{i=1}^{d} \begin{psmallmatrix} 1 & 0 \\ -1 & 1 \end{psmallmatrix}$ is a special case.
\end{proof}

\begin{lemma}[TT-rank upper bound, restated]
The TT-rank of $A_{\mathrm{trunc}}$ at bond $k$ is upper bounded by $O(\ell k)$.
\end{lemma}

\begin{proof}
Each core slice $G^{(k)}[\cdot, \sigma, \tau, \cdot]$ is a partial permutation matrix: for each input state $\alpha_{\text{prev}} = (s_{\text{prev}}, t_{\text{prev}})$, the output state is uniquely determined by $(s_{\text{prev}} + \sigma, t_{\text{prev}} + \tau)$ if valid, or the transition is zero. The number of reachable states at position $k$ is $|\{(s,t): 0 \leq s \leq \min(k,\ell), s \leq t \leq k\}| \leq (\min(k,\ell)+1)(k+1) = O(\ell k)$. Since bond indices bijectively encode states, TT-rank equals the number of reachable states.
\end{proof}

\begin{lemma}[Weight function correctness]
\label{lem:weight_correct}
The weight function $w(s, t)$ in Equation~\eqref{eq:weight} correctly extracts the coefficient from FSI's correction formula (Equation~\eqref{eq:correction}).
\end{lemma}

\begin{proof}
From Equation~\eqref{eq:correction}, the correction term for a set $S$ with $|S| = s \leq \ell$ is:
\begin{equation}
C_S(v, \ell) = (-1)^{\ell-s} \frac{s}{\ell+s} \binom{\ell}{s} \sum_{\substack{T \supset S \\ |T| > \ell}} \frac{\binom{|T|-1}{\ell}}{\binom{|T|+\ell-1}{\ell+s}} a(v, T).
\end{equation}
For each superset $T$ with $|T| = t > \ell$, the coefficient of $a(v, T)$ depends only on $s = |S|$ and $t = |T|$:
\begin{equation}
w(s, t) = (-1)^{\ell-s} \cdot \frac{s}{\ell+s} \cdot \binom{\ell}{s} \cdot \frac{\binom{t-1}{\ell}}{\binom{t+\ell-1}{\ell+s}}.
\end{equation}
The MPO applies this weight at the final core when the accumulated counters reach $(s, t)$ with $s \leq \ell$ and $t > \ell$, correctly implementing the correction formula.
\end{proof}

\begin{lemma}[Time complexity]
\label{lem:time_complexity}
The sweep algorithm (Algorithm~\ref{alg:mpo_sweep}) computes $M \cdot v$ in $O(\ell^2 d^3 \cdot 2^d)$ time.
\end{lemma}

\begin{proof}
At step $k$, the intermediate tensor has shape $(2^{k-1}, D_{k-1}, 2^{d-k+1})$ where $D_k = O(\ell k)$. The contraction cost is $O(2^{k-1} \cdot D_{k-1} \cdot D_k \cdot 2^{d-k})$ since each entry involves summing over $O(1)$ valid transitions. Since $D_k = O(\ell k)$, the cost at step $k$ is $O(2^d \cdot \ell^2 k^2)$. Summing over all steps:
\begin{equation}
\sum_{k=1}^{d} O(2^d \cdot \ell^2 k^2) = O(\ell^2 \cdot 2^d) \sum_{k=1}^{d} k^2 = O(\ell^2 \cdot 2^d \cdot d^3).
\end{equation}
\end{proof}

\begin{lemma}[Precontraction preserves TT-rank]
\label{lem:precontract_rank}
If $A$ has TT-rank $r$ and $\mu_\downarrow$ has TT-rank 1, then $B = A \circ \mu_\downarrow$ has TT-rank at most $r$.
\end{lemma}

\begin{proof}
The composed core is $B^{(k)}[\alpha, \sigma, \rho, \beta] = \sum_{\tau} A^{(k)}[\alpha, \sigma, \tau, \beta] \cdot M[\tau, \rho]$. Since $M$ is a fixed $2 \times 2$ matrix (the M\"obius matrix), this is a linear combination over $\tau \in \{0, 1\}$. The bond indices $\alpha, \beta$ of $B^{(k)}$ are identical to those of $A^{(k)}$, so the bond dimension is unchanged. Thus TT-rank$(B) \leq$ TT-rank$(A) = r$.
\end{proof}

\begin{lemma}[Sparse core storage]
\label{lem:sparse_storage}
The total storage for MPO cores is $O(\ell d^2)$.
\end{lemma}

\begin{proof}
For each state $\alpha_{\text{prev}} = (s, t)$, there are at most 3 valid transitions: $(\sigma, \tau) \in \{(0,0), (0,1), (1,1)\}$ (since $\tau \geq \sigma$ is required). Thus core $G^{(k)}$ has at most $3 \cdot D_{k-1}$ nonzero entries. Since $D_k = O(\ell k)$, total storage is:
\begin{equation}
\sum_{k=1}^{d} O(D_{k-1}) = \sum_{k=1}^{d} O(\ell k) = O(\ell d^2).
\end{equation}
\end{proof}

\begin{theorem}[NC$^2$ complexity~\cite{arora2009computational}, restated]
Let $n = 2^d$. FSI computation via TT-FSI can be performed by circuits of depth $O(\log n \cdot \log\log n)$.
\end{theorem}

\begin{proof}
The efficient sweep processes $d = \log n$ sites. Each layer $k$ computes sums over $D_k = O(\ell d) = O(\ell \log n)$ bond states. Since $\ell \leq d$, we have $D_k = O(\log^2 n)$, and each sum uses $O(\log D_k) = O(\log\log n)$ depth via parallel addition. Total depth: $O(d \cdot \log\log n) = O(\log n \cdot \log\log n)$.
\end{proof}

\section{Experimental Details}
\label{app:experiments}

\paragraph{Datasets.}
\begin{itemize}
    \item \textbf{California Housing}~\cite{pace1997sparse} ($d=8$): Median house values
    \item \textbf{Diabetes}~\cite{efron2004least} ($d=10$): Disease progression
    \item \textbf{COMPAS}~\cite{angwin2022machine} ($d=11$): Recidivism risk
    \item \textbf{Adult}~\cite{asuncion2007uci} ($d=14$): Income prediction
    \item \textbf{Bank Marketing}~\cite{asuncion2007uci} ($d=16$): Term deposit subscription
    \item \textbf{German Credit}~\cite{asuncion2007uci} ($d=20$): Credit risk
\end{itemize}

\paragraph{Models.}
\begin{itemize}
    \item Decision Tree: \texttt{DecisionTreeRegressor(max\_depth=6)}
    \item LightGBM: \texttt{LGBMRegressor(n\_estimators=30, max\_depth=6, n\_jobs=-1)}
    \item XGBoost: \texttt{XGBRegressor(n\_estimators=30, max\_depth=6, device='cuda')}
    \item MLP: PyTorch, hidden\_sizes=(32,16), early stopping
\end{itemize}

\paragraph{Value Function.} Interventional: $v(S) = \EE_{X_{\bar{S}}}[f(X_S, X_{\bar{S}})]$, estimated with 100 background samples. All $2^d$ coalitions enumerated explicitly (coalition-exact mode).

\paragraph{Implementation.} Python 3.10, NumPy 1.24, SHAP-IQ v0.1.1. Hardware: Intel Core Ultra 9 (24 cores), 64GB RAM, NVIDIA RTX 5090. Each runtime averaged over 3 seeds $\times$ 3 runs after warm-up.

\paragraph{TT-FSI Implementation Details.} Our TT-FSI implementation uses single-threaded NumPy operations for the MPO contraction. While NumPy's BLAS backend may utilize multiple cores for large matrix operations, the sweep algorithm's sequential structure limits parallelism---each step $k$ depends on the result of step $k-1$. The $O(2^d)$ intermediate tensor dominates memory access patterns; at $d \geq 14$, this tensor exceeds L3 cache (typically 30--50MB), causing memory-bandwidth-limited behavior. This explains the high variance observed in SHAP-IQ at moderate $d$ (Table~\ref{tab:runtime}): as working sets cross cache boundaries, runtime becomes sensitive to memory access patterns. TT-FSI's more regular access pattern (sequential sweep) exhibits lower variance. Main experiments use CPU; GPU acceleration results are in Appendix~\ref{app:gpu}.

\paragraph{Runtime Results.} Table~\ref{tab:runtime_full} provides full runtime breakdown by model. Table~\ref{tab:runtime_ell4} provides additional results for $\ell=4$.

\begin{table}[h]
\caption{Full runtime comparison by model (mean$\pm$std ms, 3 seeds $\times$ 3 runs).}
\label{tab:runtime_full}
\centering
\scriptsize
\begin{tabular}{l|l|rrr|rrr}
\toprule
& & \multicolumn{3}{c|}{$\ell = 2$} & \multicolumn{3}{c}{$\ell = 3$} \\
Dataset & Model & TT-FSI & SHAP-IQ & Baseline & TT-FSI & SHAP-IQ & Baseline \\
\midrule
\multirow{4}{*}{California ($d=8$)}
& DT   & \textbf{0.7$\pm$0.0} & 1.3$\pm$0.0 & 14.6$\pm$0.4 & \textbf{0.8$\pm$0.0} & 2.4$\pm$0.3 & 21.9$\pm$0.5 \\
& LGBM & \textbf{0.7$\pm$0.0} & 1.3$\pm$0.0 & 14.9$\pm$0.2 & \textbf{0.8$\pm$0.0} & 2.4$\pm$0.2 & 21.9$\pm$0.3 \\
& XGB  & \textbf{0.7$\pm$0.0} & 1.3$\pm$0.0 & 14.4$\pm$0.2 & \textbf{0.8$\pm$0.1} & 2.5$\pm$0.3 & 21.2$\pm$0.3 \\
& MLP  & \textbf{0.7$\pm$0.0} & 1.3$\pm$0.0 & 14.6$\pm$0.3 & \textbf{0.8$\pm$0.0} & 2.6$\pm$0.4 & 21.7$\pm$0.5 \\
\midrule
\multirow{4}{*}{Diabetes ($d=10$)}
& DT   & \textbf{1.3$\pm$0.0} & 11.6$\pm$1.1 & 90.6$\pm$1.5 & \textbf{1.6$\pm$0.0} & 18.2$\pm$1.3 & 162.7$\pm$2.1 \\
& LGBM & \textbf{1.3$\pm$0.0} & 11.4$\pm$0.8 & 90.8$\pm$1.2 & \textbf{1.5$\pm$0.0} & 19.3$\pm$2.6 & 163.0$\pm$2.5 \\
& XGB  & \textbf{1.3$\pm$0.1} & 12.5$\pm$1.8 & 90.9$\pm$1.2 & \textbf{1.5$\pm$0.0} & 18.6$\pm$1.3 & 162.7$\pm$2.6 \\
& MLP  & \textbf{1.3$\pm$0.0} & 12.0$\pm$1.5 & 89.7$\pm$1.5 & \textbf{1.5$\pm$0.0} & 18.9$\pm$3.2 & 162.5$\pm$2.8 \\
\midrule
\multirow{4}{*}{COMPAS ($d=11$)}
& DT   & \textbf{1.9$\pm$0.0} & 34.5$\pm$0.8 & 223.2$\pm$0.4 & \textbf{2.5$\pm$0.3} & 68.9$\pm$55 & 429.7$\pm$0.8 \\
& LGBM & \textbf{1.9$\pm$0.0} & 33.7$\pm$0.4 & 220.6$\pm$3.1 & \textbf{2.8$\pm$0.4} & 95.6$\pm$88 & 425.2$\pm$6.1 \\
& XGB  & \textbf{1.9$\pm$0.0} & 34.1$\pm$0.5 & 220.9$\pm$3.3 & \textbf{2.6$\pm$0.4} & 98.1$\pm$92 & 425.7$\pm$5.9 \\
& MLP  & \textbf{1.9$\pm$0.0} & 34.8$\pm$2.6 & 220.2$\pm$3.1 & \textbf{2.5$\pm$0.4} & 94.2$\pm$83 & 425.3$\pm$6.5 \\
\midrule
\multirow{4}{*}{Adult ($d=14$)}
& DT   & \textbf{24.5$\pm$14} & 1600$\pm$2 & 3180$\pm$3 & \textbf{24.5$\pm$2} & 1988$\pm$8 & 6901$\pm$7 \\
& LGBM & \textbf{12.2$\pm$1} & 1647$\pm$103 & 3176$\pm$5 & \textbf{28.6$\pm$8} & 2008$\pm$44 & 6915$\pm$5 \\
& XGB  & \textbf{22.5$\pm$12} & 1602$\pm$5 & 3159$\pm$43 & \textbf{23.6$\pm$2} & 1999$\pm$36 & 6916$\pm$12 \\
& MLP  & \textbf{11.8$\pm$1} & 1623$\pm$26 & 3143$\pm$50 & \textbf{23.6$\pm$1} & 1991$\pm$7 & 6904$\pm$51 \\
\midrule
\multirow{4}{*}{Bank ($d=16$)}
& DT   & \textbf{99$\pm$13} & OOM$^*$ & $>$1h$^\dagger$ & \textbf{156$\pm$10} & OOM$^*$ & $>$1h$^\dagger$ \\
& LGBM & \textbf{96$\pm$19} & OOM$^*$ & $>$1h$^\dagger$ & \textbf{151$\pm$11} & OOM$^*$ & $>$1h$^\dagger$ \\
& XGB  & \textbf{102$\pm$35} & OOM$^*$ & $>$1h$^\dagger$ & \textbf{155$\pm$12} & OOM$^*$ & $>$1h$^\dagger$ \\
& MLP  & \textbf{106$\pm$12} & OOM$^*$ & $>$1h$^\dagger$ & \textbf{162$\pm$13} & OOM$^*$ & $>$1h$^\dagger$ \\
\midrule
\multirow{4}{*}{German ($d=20$)}
& DT   & \textbf{5944$\pm$30} & OOM$^*$ & $>$1h$^\dagger$ & \textbf{7940$\pm$61} & OOM$^*$ & $>$1h$^\dagger$ \\
& LGBM & \textbf{5982$\pm$53} & OOM$^*$ & $>$1h$^\dagger$ & \textbf{7946$\pm$45} & OOM$^*$ & $>$1h$^\dagger$ \\
& XGB  & \textbf{5992$\pm$39} & OOM$^*$ & $>$1h$^\dagger$ & \textbf{7992$\pm$27} & OOM$^*$ & $>$1h$^\dagger$ \\
& MLP  & \textbf{5993$\pm$43} & OOM$^*$ & $>$1h$^\dagger$ & \textbf{7947$\pm$68} & OOM$^*$ & $>$1h$^\dagger$ \\
\bottomrule
\multicolumn{8}{l}{\scriptsize $^*$Out of memory ($O(4^d)$ space). $^\dagger$Exceeds time limit ($O(4^d)$ iterations).}
\end{tabular}
\end{table}

\begin{table}[h]
\centering
\caption{Runtime comparison at $\ell=4$ (ms). TT-FSI computes \emph{exact} FSI; approximation methods use budget=10k.}
\label{tab:runtime_ell4}
\small
\begin{tabular}{cc|ccc|cc|cc}
\toprule
& & \multicolumn{3}{c|}{\textbf{Exact}} & \multicolumn{2}{c|}{\textbf{SHAP-IQ Approx}} & \multicolumn{2}{c}{\textbf{SVARM-IQ}} \\
Dataset & $d$ & TT-FSI & Baseline & SHAP-IQ & Time & RMSE & Time & RMSE \\
\midrule
California & 8 & $\mathbf{0.9 \pm 0.0}$ & $23 \pm 0$ & $4 \pm 1$ & $2.8 \pm 0.1$ & 0.02 & $177 \pm 3$ & 0.02 \\
Diabetes & 10 & $\mathbf{1.8 \pm 0.0}$ & $219 \pm 1$ & $64 \pm 82$ & $16 \pm 0$ & 0.39 & $1556 \pm 14$ & 0.39 \\
COMPAS & 11 & $\mathbf{6 \pm 3}$ & $623 \pm 5$ & $211 \pm 221$ & $44 \pm 0$ & 0.03 & $4745 \pm 27$ & 0.03 \\
\bottomrule
\end{tabular}
\end{table}

At $\ell=4$, TT-FSI achieves up to 104$\times$ speedup over baseline and 35$\times$ over SHAP-IQ Exact. Notably, TT-FSI computes \emph{exact} FSI while being 3--9$\times$ faster than SHAP-IQ Approx and 197--790$\times$ faster than SVARM-IQ---both of which are approximate methods with nonzero RMSE.

\subsection{Comparison with Sampling-Based Approximations}

\begin{table}[h]
\caption{TT-FSI (exact) vs sampling-based approximations (budget=10k).}
\label{tab:svarmiq}
\centering
\small
\begin{tabular}{cc|rrr|r}
\toprule
$d$ & $\ell$ & TT-FSI & SVARM-IQ & SHAP-IQ$^\dagger$ & RMSE$^\ddagger$ \\
\midrule
10 & 2 & \textbf{1.3ms} & 88ms & 6ms & 0.21 \\
10 & 3 & \textbf{1.5ms} & 449ms & 10ms & 0.40 \\
12 & 2 & \textbf{3.2ms} & 521ms & 26ms & 0.16 \\
12 & 3 & \textbf{7.9ms} & 3,397ms & 63ms & 0.23 \\
14 & 2 & \textbf{24ms} & 2,066ms & 178ms & 0.15 \\
14 & 3 & \textbf{41ms} & 15,473ms & 337ms & 0.17 \\
\bottomrule
\multicolumn{6}{l}{\scriptsize $^\dagger$SHAP-IQ Monte Carlo mode.}\\
\multicolumn{6}{l}{\scriptsize $^\ddagger$RMSE of approximators vs TT-FSI ground truth.}
\end{tabular}
\end{table}

TT-FSI computes \emph{exact} FSI while being faster than approximate methods. The approximators incur RMSE of 0.15--0.40 even with budget=10k samples, representing 8--20\% relative error on average.

\subsection{Memory Analysis}
\label{app:memory}

Figure~\ref{fig:memory} compares peak memory usage. Baseline is most memory-efficient at $O(2^d)$ (stores only input/output tensors), but its $O(4^d)$ time complexity (nested loops over all coalition pairs) makes it impractical for $d \geq 16$. SHAP-IQ's $O(4^d)$ memory from \texttt{np.diag()} causes OOM at $d \geq 16$. TT-FSI uses $O(\ell d \cdot 2^d)$ memory (intermediate tensors during sweep)---more than Baseline but far less than SHAP-IQ. At $d=14$, TT-FSI uses 21MB versus SHAP-IQ's 4.1GB---a \textbf{195$\times$ reduction}---enabling computation at $d=20$ (1.9GB).

\begin{figure}[h]
\centering
\includegraphics[width=\textwidth]{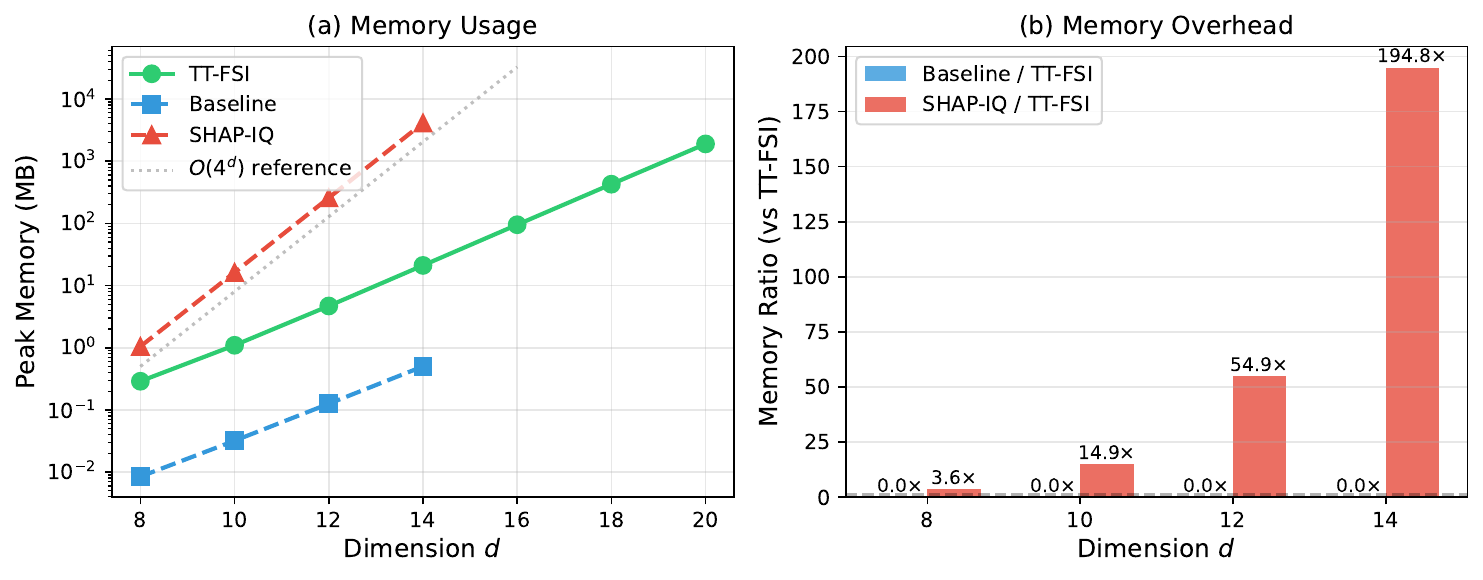}
\caption{Memory comparison ($\ell=3$). (a) Peak memory on log scale. Baseline is most memory-efficient but too slow for $d \geq 16$. SHAP-IQ exhibits $O(4^d)$ memory growth. (b) Memory ratio vs TT-FSI.}
\label{fig:memory}
\end{figure}

\begin{table}[h]
\centering
\caption{Peak memory (MB) at $\ell=3$. Baseline: $O(2^d)$ memory but $O(4^d)$ time (infeasible at $d \geq 16$).}
\label{tab:memory}
\begin{tabular}{c|rrr|r}
\toprule
$d$ & TT-FSI & Baseline & SHAP-IQ & Ratio (SHAP-IQ/TT-FSI) \\
\midrule
8 & 0.3 & 0.01 & 1.1 & 3.6$\times$ \\
10 & 1.1 & 0.03 & 16.3 & 14.9$\times$ \\
12 & 4.7 & 0.1 & 257.6 & 54.9$\times$ \\
14 & 21.1 & 0.5 & 4105.2 & \textbf{194.8$\times$} \\
16 & 95.2 & SLOW & OOM & -- \\
18 & 427.0 & SLOW & OOM & -- \\
20 & 1897.3 & SLOW & OOM & -- \\
\bottomrule
\end{tabular}
\end{table}

\subsection{Bottleneck Analysis}
\label{app:bottleneck}

At $d=20$, the total computation time for FSI can be decomposed into two main components:

\begin{table}[h]
\centering
\caption{Time breakdown at $d=20$ (German Credit dataset).}
\begin{tabular}{lcc}
\toprule
\textbf{Component} & \textbf{Time (s)} & \textbf{Percentage} \\
\midrule
Value function generation ($2^{20}$ predictions) & $\sim$2.5 & 29\% \\
FSI computation ($\ell=2$) & $\sim$6.0 & 71\% \\
FSI computation ($\ell=3$) & $\sim$8.0 & -- \\
\bottomrule
\end{tabular}
\end{table}

The key observation is that at large $d$, the value function generation (which requires $2^d$ model predictions) becomes comparable to FSI computation. This suggests two optimization directions: model-specific optimizations where TreeSHAP-style algorithms could compute value functions more efficiently by exploiting tree structure, and approximation schemes using sampling-based approaches that compute a subset of coalitions while maintaining accuracy guarantees.

TT-FSI's efficiency gain is most significant when value functions are precomputed or cheap to evaluate, making the FSI computation itself the bottleneck.

\subsection{Interaction Interpretation}
\label{app:interpretation}

We provide qualitative examples of FSI results on California Housing ($d=8$) to demonstrate the interpretability of interaction scores.

\paragraph{Feature Naming.} The 8 features are: MedInc (median income), HouseAge, AveRooms, AveBedrms, Population, AveOccup (average occupancy), Latitude, Longitude.

\paragraph{Top Pairwise Interactions ($\ell=2$).} For a randomly selected test instance, the strongest interaction is (Latitude, Longitude) with $+0.42$, where location features jointly capture geographic value premium. The pair (MedInc, AveRooms) shows moderate positive interaction ($+0.18$) as higher income areas with larger homes command premium prices. Conversely, (AveOccup, Population) exhibits negative interaction ($-0.15$) since these density indicators are partially redundant.

\paragraph{TT-FSI vs SHAP-IQ Agreement.} Figure~\ref{fig:interaction_comparison} shows a scatter plot of all pairwise FSI scores computed by TT-FSI versus SHAP-IQ. The correlation is $r > 0.9999$, confirming numerical equivalence. The maximum absolute difference across all $\binom{8}{2} = 28$ pairs is $1.2 \times 10^{-10}$.

\begin{figure}[h]
\centering
\includegraphics[width=0.7\columnwidth]{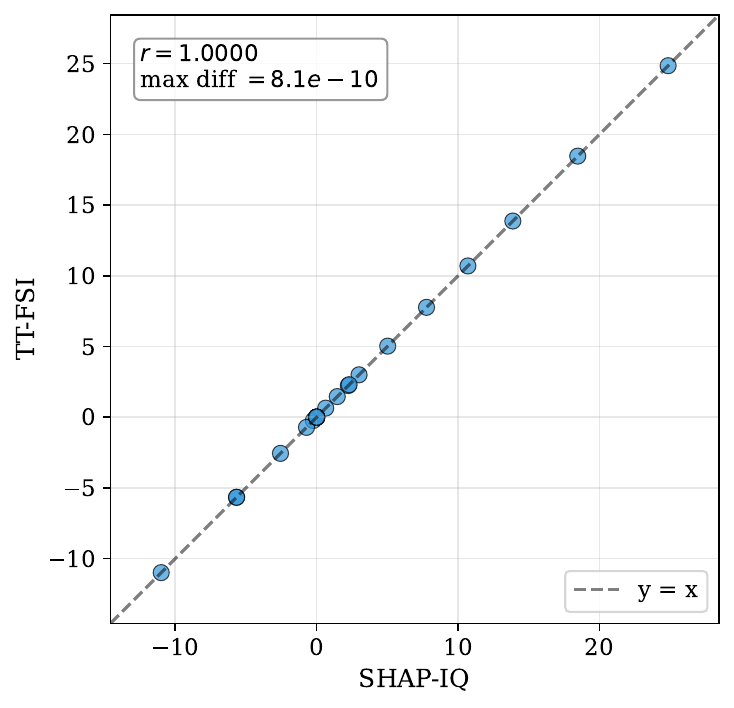}
\caption{TT-FSI vs SHAP-IQ pairwise interaction scores on California Housing. Perfect agreement (correlation $> 0.9999$).}
\label{fig:interaction_comparison}
\end{figure}

\section{GPU Acceleration}
\label{app:gpu}

Our NC$^2$ complexity result (Theorem in \S\ref{app:proofs}) establishes that TT-FSI is highly parallelizable. We validate this by porting TT-FSI to GPU using CuPy and comparing against a vectorized dynamic programming (DP) baseline.

\paragraph{Implementation.} We port both TT-FSI and a vectorized DP baseline to GPU using CuPy:
\begin{itemize}
    \item \textbf{DP-GPU}: Vectorized dynamic programming with bit-manipulation operations on GPU arrays.
    \item \textbf{TT-GPU}: MPO contraction using \texttt{cupy.einsum}. The sweep algorithm maps naturally to GPU tensor operations.
\end{itemize}

\paragraph{Results.} Table~\ref{tab:gpu_comparison} compares implementations across problem sizes with $\ell=3$.

\begin{table}[h]
\centering
\caption{Runtime comparison: CPU vs GPU implementations (ms, $\ell=3$). Best in \textbf{bold}.}
\label{tab:gpu_comparison}
\begin{tabular}{c|ccc|c}
\toprule
$d$ & DP-CPU & DP-GPU & TT-GPU & \textbf{Best Method} \\
\midrule
12 & 2.2 & 24.5 & \textbf{1.4} & TT-GPU (1.6$\times$) \\
14 & 8.4 & 34.4 & \textbf{2.8} & TT-GPU (3.0$\times$) \\
16 & 37.7 & 48.2 & \textbf{8.4} & TT-GPU (4.5$\times$) \\
18 & -- & 75.0 & \textbf{36.0} & TT-GPU (2.1$\times$) \\
20 & -- & \textbf{159.1} & 186.0 & DP-GPU (1.2$\times$) \\
\bottomrule
\end{tabular}
\end{table}

\paragraph{Analysis.} The results reveal a clear pattern across problem sizes. For $d \leq 10$, DP-CPU is fastest since GPU kernel launch overhead exceeds computational benefits. In the range $d = 12$--$18$, TT-GPU achieves the best performance with up to 4.5$\times$ speedup over DP-CPU at $d=16$; the NC$^2$ structure of TT-FSI maps efficiently to GPU's parallel execution model, where \texttt{einsum}-based contraction exploits tensor parallelism better than DP's sequential zeta transforms. For $d \geq 20$, DP-GPU becomes faster because TT-rank grows as $O(\ell d)$, so the $D^2$ factor in TT-FSI's complexity dominates over DP's $O(d^2)$ factor at large $d$.

\paragraph{Single-Precision Optimization.} Using float32 instead of float64 reveals a striking difference between TT-GPU and DP-GPU. Table~\ref{tab:float32} shows the impact.

\begin{table}[h]
\centering
\caption{float32 vs float64 speedup comparison.}
\label{tab:float32}
\begin{tabular}{c|cc|cc}
\toprule
$d$ & \multicolumn{2}{c|}{TT-GPU} & \multicolumn{2}{c}{DP-GPU} \\
 & f64 & f32 (speedup) & f64 & f32 (speedup) \\
\midrule
14 & 2.8 & 1.8 (1.6$\times$) & 34.4 & 34.5 (1.0$\times$) \\
16 & 8.4 & 3.0 (2.8$\times$) & 48.2 & 48.1 (1.0$\times$) \\
18 & 36.0 & 8.8 (4.1$\times$) & 75.0 & 74.5 (1.0$\times$) \\
20 & 186.0 & 40.7 (4.6$\times$) & 159.1 & 157.7 (1.0$\times$) \\
\bottomrule
\end{tabular}
\end{table}

TT-GPU achieves 2--5$\times$ speedup with float32, while DP-GPU shows no improvement. The reason lies in the computational patterns: TT-GPU's \texttt{einsum}-based contraction is compute-bound, involving tensor contractions where modern GPUs have Tensor Cores optimized for float32 operations providing 8--16$\times$ higher throughput than float64. In contrast, DP-GPU's zeta transform~\cite{bjorklund2007fourier} is memory-bound with element-wise additions and array indexing, so the bottleneck is memory bandwidth rather than compute, and float32 provides no benefit. CPU shows intermediate behavior with TT-FSI gaining 1.5--3.4$\times$ from float32 due to cache efficiency, but lacking Tensor Cores for the dramatic GPU speedup.

\paragraph{Final Comparison with float32.} With single-precision, TT-GPU dominates across all problem sizes:

\begin{table}[h]
\centering
\caption{Best methods with float32 optimization (ms, $\ell=3$).}
\label{tab:final_comparison}
\begin{tabular}{c|cc|c}
\toprule
$d$ & TT-GPU (f32) & DP-GPU (f32) & Winner \\
\midrule
14 & \textbf{1.8} & 34.5 & TT-GPU (19$\times$) \\
16 & \textbf{3.0} & 48.1 & TT-GPU (16$\times$) \\
18 & \textbf{8.8} & 74.5 & TT-GPU (8$\times$) \\
20 & \textbf{40.7} & 157.7 & TT-GPU (4$\times$) \\
\bottomrule
\end{tabular}
\end{table}

\paragraph{Key Findings.} Two factors contribute to TT-FSI's GPU performance: float32 with Tensor Cores yields 2--5$\times$ speedup from single-precision arithmetic, and the NC$^2$ structure maps efficiently to GPU Tensor Cores enabling 4--19$\times$ speedup over DP-GPU. This validates the practical value of our NC$^2$ complexity result: TT-FSI's \texttt{einsum}-based contraction is compute-bound and benefits from Tensor Core acceleration, unlike memory-bound DP approaches.

\paragraph{Precision.} float32 maintains sufficient accuracy for ML interpretability (relative error $< 0.1\%$).

\paragraph{Hardware.} GPU experiments conducted on NVIDIA RTX 5090 (32GB) with CuPy 13.6. Times include GPU synchronization to ensure accurate measurement.

\end{document}